\documentclass[11pt]{article}

\usepackage[T1]{fontenc}
\usepackage{lmodern}
\usepackage{microtype}
\usepackage[letterpaper,margin=1in]{geometry}
\usepackage{amsmath,amssymb,amsthm}
\usepackage[round]{natbib}
\usepackage{graphicx}
\usepackage{url}
\usepackage{booktabs}
\usepackage{longtable}
\usepackage[hidelinks]{hyperref}

\theoremstyle{plain}

\newtheorem{proposition}{Proposition}

\theoremstyle{definition}

\newcommand{\sectionref}[1]{Section~\ref{#1}}
\newcommand{\appendixref}[1]{Appendix~\ref{#1}}
\newcommand{\figureref}[1]{Figure~\ref{#1}}

\makeatletter
\newcommand{\floatconts}[3]{%
  \centering #3\par
  \begingroup
    \let\csnn@caption\caption
    \renewcommand{\caption}[1]{\csnn@caption{##1}\label{#1}}%
    #2%
  \endgroup}
\makeatother

\newenvironment{keywords}{\par\medskip\noindent\textbf{Keywords:}\ }{\par}

\newcommand{\Cl}{\mathrm{Cl}}
\newcommand{\Spin}{\mathrm{Spin}}

\newcommand{\Ad}{\mathrm{Ad}}

\newcommand{\FinVect}{\mathbf{FinVect}}
\newcommand{\End}{\mathrm{End}}
\newcommand{\Lap}{\mathcal{L}}
\newcommand{\rev}[1]{\widetilde{#1}}
\newcommand{\ip}[2]{\langle #1, #2 \rangle}
\newcommand{\grade}[2]{[ #1 ]_{#2}}
\newcommand{\norm}[1]{\lVert #1 \rVert}

\title{Clifford Sheaf Neural Networks}

\author{%
  Kotaro Kamiya\thanks{Corresponding author: \texttt{k.kamiya@syntheticgestalt.com}}
  \qquad
  Joel Nicholls\thanks{Work done while at SyntheticGestalt KK.}\\[4pt]
  SyntheticGestalt KK}
\date{}

\begin{document}

\maketitle

\begin{abstract}
We introduce the Clifford Sheaf Neural Network (CSNN), an equivariant sheaf neural network
for geometric graphs that places a Clifford algebra on each stalk of a cellular sheaf and
transports multivector features along edges. The canonical choice of restriction map for sheaves with algebra-valued stalks is algebra homomorphism. Adding the constraint of equivariance, the naive choice becomes versor conjugation. However, versor conjugation is expressively weak, so we drop algebra homomorphism and arrive at the K-term sandwich. The resulting sheaf Laplacian is positive semidefinite by
construction, needs no versor constraint, and still mixes grades. Our main contribution
characterizes the resulting family of restriction maps along three axes: which grades a map
couples, how much of the endomorphism space it reaches, and how well it is conditioned.
The K-term sandwich spans half of the endomorphism space, and in $\Cl(3,0,0)$ it corresponds to the maps that commute with
the central pseudoscalar. The number of terms controls expressivity.
CSNN is the reversion member, a first-order model by construction and the
grade-mixing corner of this family, developed as a sheaf construction for graph-level
equivariant regression.
\end{abstract}


\section{Introduction}
\label{sec:intro}

Many learning problems live on geometric graphs, where a spatial position is assigned to each node.
Such problems reward architectures that respect the symmetries
of space. Two lines of work
each supply a powerful inductive bias. \emph{Clifford} (geometric) algebra represents
scalars, vectors, and higher-grade quantities as a single multivector and equips them
with an equivariant geometric product, so a network manipulates geometric features of
every grade under one operation \citep{ruhe2023cgenn,brehmer2023gatr}. \emph{Cellular sheaves} ask neighbours to agree only after a learned transport, not
directly. The sheaf Laplacian, which generalizes the graph Laplacian, encodes this and lets
message passing move beyond homophily \citep{hansen2020sheaf,bodnar2022nsd}.

The two can be combined by placing a Clifford algebra on each stalk and transporting
multivectors along edges. But the obvious construction runs into two difficulties. Suppose
the restriction maps are Clifford-algebra automorphisms, i.e.\ conjugation by a group
element. First, the sheaf Laplacian behaves well only when the restriction maps are
co-isometries, and that holds only for \emph{versors}, an awkward constraint to maintain
during learning. Second, conjugation preserves grades: transport cannot move information
between them, and wastes much of what the multivector offers.

We resolve both by changing the learnable family of restriction maps. Generally, the K-term sandwich, and as a specific case
the \emph{normalized reversion sandwich} $s_c(x)=c\,x\,\rev{c}/q(c)$ with $c$ an unconstrained
multivector. Applied on both sides of each edge, it gives a sheaf
Laplacian that stays positive semidefinite \emph{by construction,
with no versor constraint}, and still transports information across grades. We call the
resulting architecture the Clifford Sheaf Neural Network (CSNN).

Our second contribution characterizes restriction maps along
three axes: grade coupling, reach within the endomorphism space, and conditioning. CSNN uses
the reversion member, which mixes scalar and vector grades but stays confined to them and has
single-ratio conditioning. It is first-order by construction and, to our knowledge, the first
sheaf construction aimed at graph-level equivariant regression. The result is both a
principled construction and a map of its design space.

\section{Background}
\label{sec:background}

\subsection{Clifford Algebras}
\label{sec:clifford}

A Clifford algebra assembles several tensor grades into a single multivector. In
$\Cl(3,0,0)$, the grades are scalar, vector, bivector, and pseudoscalar, so a
node's scalar feature and its $3$D position live in one element, coupled equivariantly by the
geometric product.

Given a real vector space $V=\mathbb{R}^n$ with a quadratic form $Q$ of signature
$(p,q,r)$, the Clifford algebra $\Cl(p,q,r)$ is the quotient of the tensor algebra of
$V$ by the relations $v^2=Q(v)$ for all $v\in V$. Its multiplication is the \emph{geometric
product}. We denote the geometric product of multivectors $x$ and $y$ as $xy$.
On an orthogonal basis $\{e_1,\dots,e_n\}$ the defining relations read
\begin{equation}
  e_i e_j + e_j e_i = 2\,g_{ij},\qquad
  (g_{ij})=\mathrm{diag}(\underbrace{+1,\dots}_{p},\underbrace{-1,\dots}_{q},\underbrace{0,\dots}_{r}),
\end{equation}
and $\dim\Cl(p,q,r)=2^n$ with $n=p+q+r$. The algebra is graded,
\begin{equation}
  \Cl(p,q,r)=\bigoplus_{k=0}^{n}\Cl^{(k)}(p,q,r),\qquad \dim\Cl^{(k)}=\binom{n}{k},
\end{equation}
where $\Cl^{(k)}$ is spanned by the $k$-blades, which are products of $k$ orthogonal basis vectors.
We write $[x]_k$ for the projection of a multivector $x$ onto its grade-$k$ component. In particular, $[x]_0$ is its scalar
part.

\emph{Reversion} $\rev{x}$ reverses the order of vectors within each blade, acting as
$(-1)^{k(k-1)/2}$ on grade $k$. The \emph{grade involution} $\alpha(x)$ acts as $(-1)^k$ on
grade $k$, flipping the sign of the odd grades. Reversion defines the \emph{Clifford scalar
product}
\begin{equation}
  \ip{x}{y} := \grade{\rev{x}\,y}{0},
  \label{eq:scalar-product}
\end{equation}
the scalar part of $\rev{x}\,y$, with quadratic form $q(x):=\ip{x}{x}=\grade{\rev{x}\,x}{0}$.
This is a \emph{symmetric bilinear form}, not in general an inner product. In the Euclidean
$\Cl(3,0,0)$ it is a genuine inner product with the blades orthonormal. In
the degenerate $\Cl(3,0,1)$ the Clifford scalar product is only positive semidefinite, with a nontrivial radical
spanned by the blades containing the null generator $e_0$. Positive semidefiniteness is what
the sheaf Laplacian needs (\sectionref{sec:sheaves}), and we use the form throughout to
define adjoints.

\subsection{Cellular Sheaves and the Sheaf Laplacian}
\label{sec:sheaves}

\begin{figure}[t]
\floatconts
  {fig:cellposet}
  {\caption{The cell poset for a single edge, viewed as the base space. The two vertices lie
   below the edge they bound, and the arrows are the incidence relations along which
   restriction maps transport features.}}
  {\includegraphics[width=0.32\linewidth]{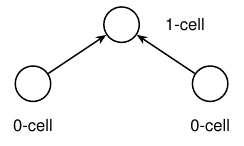}}
\end{figure}

We view an undirected graph as a cell poset with vertices below the edges they bound
(\figureref{fig:cellposet}).
A cellular sheaf assigns to each cell a vector space, its \emph{stalk}, and to each incident
vertex--edge pair a linear \emph{restriction map}. The admissible family of restriction maps is a modelling choice. Taking the target to be
the category of Clifford algebras would force the maps to be algebra automorphisms, which is
too restrictive, so we take the target to be $\FinVect$, the category of finite-dimensional
vector spaces. Restriction maps are then arbitrary linear maps between stalks, and the network exploits the multivector structure internally
rather than as part of the sheaf formalism.

The sheaf Laplacian is $\Lap=\delta^{\ast}\delta$, where the coboundary $\delta$ maps node
features to edge features,
\begin{equation}
  (\delta v)_e = f_{j\to e}(v_j) - f_{i\to e}(v_i),
  \label{eq:coboundary-generic}
\end{equation}
with $f_{k\to e}$ the restriction map from node $k$ into edge $e$. We take its adjoint
$\delta^{\ast}$ with respect to the Clifford scalar product on stalks,
$\ip{\delta v}{\omega}_E=\ip{v}{\delta^{\ast}\omega}_V$, where $\ip{\cdot}{\cdot}_V$ and
$\ip{\cdot}{\cdot}_E$ aggregate the per-stalk product over nodes and over oriented edges.
As a Gram operator the Laplacian is positive semidefinite by construction:
\begin{equation}
  \ip{v}{\Lap v}_V = \ip{\delta v}{\delta v}_E = \norm{\delta v}^2 \ge 0.
  \label{eq:psd}
\end{equation}
Two facts give~\eqref{eq:psd}. The Gram \emph{structure} $\Lap=\delta^{\ast}\delta$ yields
$\ip{v}{\Lap v}_V=\norm{\delta v}^2$ whenever the restriction maps admit an adjoint under the
stalk form, with no orthogonality or versor constraint. In the non-degenerate $\Cl(3,0,0)$
every linear map does; in the degenerate $\Cl(3,0,1)$, where an adjoint exists only for maps
preserving the radical, the normalized reversion sandwiches do (\sectionref{sec:laplacian}).
The \emph{non-negativity} then holds because the Clifford scalar product is positive
semidefinite in both signatures; the orthogonal maps of connection-Laplacian
sheaves~\citep{barbero2022sheaf} are one special case. The energy $\norm{\delta v}^2$ is the
sheaf Dirichlet energy~\citep{bodnar2022nsd}, measuring disagreement \emph{after transport},
so neighbours need agree only up to a learned transport, which carries sheaf models beyond
homophily.

\section{The Sandwich Restriction-Map Family}
\label{sec:theory}

With the sheaf Laplacian fixed (\sectionref{sec:sheaves}), the modelling choice is the
restriction map. We build restriction maps from the \emph{sandwich}: the single map
$x\mapsto a\,x\,b$ and, more generally, the $K$-term sum $\sum_{k=1}^{K}a_k\,x\,b_k$ with
multivector coefficients $a_k,b_k$. Every such map admits an adjoint under the Clifford
scalar product, so the sheaf Laplacian~\eqref{eq:psd} is positive semidefinite for the
whole family. The member we ultimately use is the \emph{normalized reversion sandwich}
\begin{equation}
  s_c(x) := c\,x\,\rev{c}\,/\,q(c),
  \label{eq:sandwich}
\end{equation}
a linear map in $x$ parameterized by a single multivector $c$, the \emph{restriction
element}, defined whenever $q(c)\neq0$, which is automatic in the positive-definite $\Cl(3,0,0)$. When $c$ is a versor (an element of
the Clifford group $\Gamma$), $s_c$ reduces to the grade-preserving conjugation $\Ad_c$. For a general $c$ it is not grade-preserving. This
section develops the family from the general sandwich down to $s_c$ along axes of grade
coupling, reachability, and conditioning; all proofs are in \appendixref{app:proofs}.

\subsection{Equivariance}
\label{sec:equivariance}

Sandwich maps are $\Spin$-equivariant when their factors transform with the input.
\begin{proposition}[Equivariance of the sandwich family]
\label{prop:equiv}
For any unit even versor $R\in\Spin$, any multivectors $a_k,b_k,x$, and the $K$-term
sandwich $T(x)=\sum_k a_k\,x\,b_k$,
\[
  \textstyle\sum_k (R a_k \rev{R})\,(R x \rev{R})\,(R b_k \rev{R}) = R\,T(x)\,\rev{R}.
\]
In particular the normalized reversion sandwich obeys
$s_{Rc\rev{R}}(Rx\rev{R}) = R\,s_c(x)\,\rev{R}$.
\end{proposition}

In the layer, the network generates the restriction elements covariantly from the inputs
(\sectionref{sec:generate}), so the transport moves with the data. A CSNN layer composes
the sandwich transport, the sheaf Laplacian, and
$\Spin$-equivariant gates and feed-forward maps. As a composition of equivariant maps it
is equivariant, so the network is $SO(3)$-equivariant in $\Cl(3,0,0)$ and
$SE(3)$-equivariant in $\Cl(3,0,1)$.

\subsection{Three Axes: Grade Coupling, Reachability, and Conditioning}
\label{sec:family}

Every member of the sandwich family is equivariant, whatever its form
(\sectionref{sec:equivariance}). Which linear maps it can realize, once we lift the versor
constraint, is a separate question, answered by three axes: which grades a map couples, how
much of the endomorphism space it reaches, and how well it is conditioned. CSNN's reversion
map is one corner of this family, with conditioning set by a single ratio.

\begin{figure}[t]
\floatconts
  {fig:coupling}
  {\caption{Grade coupling of the sandwich family in $\Cl(3,0,0)$. Nodes are the grades
   $0$--$3$; an arc joins two grades that a map couples. Conjugation is grade-preserving; the
   reversion sandwich $s_c$ couples $\{0,1\}$ and $\{2,3\}$; the $\gamma$-sandwich couples
   $\{0,3\}$ and $\{1,2\}$; a general or $K$-term sandwich couples all six pairs. Reachable
   dimensions are listed in \figureref{fig:hierarchy}.}}
  {\includegraphics[width=\linewidth]{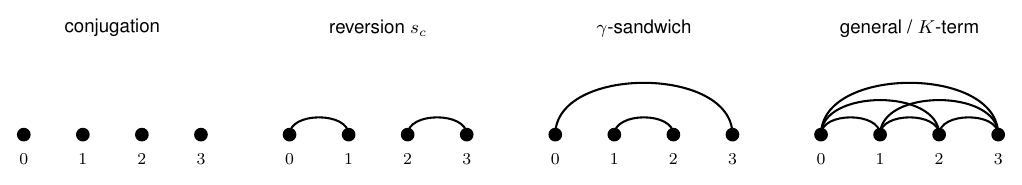}}
\end{figure}

\paragraph{Grade coupling.}
The first axis asks which grades a map can mix (\figureref{fig:coupling}), and three involutions answer it. Besides
reversion $\rev{x}=\beta(x)$ and the grade involution $\alpha$, $\Cl(3,0,0)$ carries the
Clifford conjugation $\gamma=\alpha\circ\beta$. Each
splits the algebra into $\pm1$ eigenspaces, grouping the grades into pairs: $\beta$ into
$\{0,1\}$ and $\{2,3\}$, $\gamma$ into $\{0,3\}$ and $\{1,2\}$, and $\alpha$ into
$\{0,2\}$ and $\{1,3\}$. A map preserving an involution's eigenspaces cannot couple
grades across the corresponding boundary.

\begin{proposition}[Grade confinement]
\label{prop:confine}
The normalized reversion sandwich $s_c$ preserves the $\beta$-eigenspaces for every $c$, as does any composition of
reversion sandwiches. In $\Cl(3,0,0)$ these eigenspaces are $\{0,1\}$ and $\{2,3\}$, so a
network seeded in grade $0$ stays confined to grades $\{0,1\}$ at every depth.
\end{proposition}

The reversion sandwich thus couples only $\{0\!\leftrightarrow\!1\}$ and
$\{2\!\leftrightarrow\!3\}$: it mixes grades, unlike grade-preserving conjugation, but
not fully. A single general sandwich $a\,x\,b$ with independent $a,b$ preserves no
involution and couples all six grade pairs. This ``full grade mixing'' is a property of
one map's support, distinct from the reachability question we turn to next.

\paragraph{Reachable maps and the reachability ceiling.}
The second axis asks how much of the endomorphism space the family reaches. In $\Cl(3,0,0)$
the pseudoscalar $e_{123}$ is central and squares to $-1$, so $J(x)=e_{123}x$ satisfies
$J^2=-\mathrm{Id}$ and splits $\End_{\mathbb{R}}$ (dimension $64$) into two halves of dimension $32$: the
maps that \emph{commute} with $J$ and those that \emph{anticommute}.

\begin{proposition}[Reachability ceiling in $\Cl(3,0,0)$]
\label{prop:clin}
In $\Cl(3,0,0)$, every $K$-term sandwich $\sum_k a_k\,x\,b_k$ commutes with $J$, so the
entire sandwich family lies in the commuting half. Reaching the anticommuting half requires
the grade involution, $x\mapsto a\,\alpha(x)\,b$.
\end{proposition}

Within the reachable half the number of terms strictly controls reach: the reachable
dimension climbs $14,24,30,32$ in $\Cl(3,0,0)$ (ceiling $32$ at $K=4$) and
$31,60,87,112,\dots$ in $\Cl(3,0,1)$ (ceiling $128$ at $K=8$).

\paragraph{A signature-independent obstruction.}
Proposition~\ref{prop:clin} excludes $\alpha$ from the sandwich family in $\Cl(3,0,0)$ through
the centrality of the pseudoscalar, which fails in the degenerate $\Cl(3,0,1)$; a
grade-weighted trace nonetheless settles both signatures at once, since no finite sum
$\sum_k a_k\,x\,b_k$ can equal $\alpha$ (\appendixref{app:proofs}). So the sandwich family
fills exactly the pseudoscalar-commuting half of $\End_{\mathbb{R}}$ in $\Cl(3,0,0)$, and in
$\Cl(3,0,1)$ a direct computation gives $128$ of $256$, which $\alpha$ enlarges only to $192$. This concerns
which maps are realizable, not symmetry: every member is equivariant
(\sectionref{sec:equivariance}), with or without $\alpha$.

\paragraph{Chirality.}
Telling a shape from its mirror image needs a parity-odd grade-$3$ feature. By
Proposition~\ref{prop:confine} the reversion sandwich never populates grade $3$ from scalar
and vector inputs, so it is chirality-insensitive, whereas a general or $K$-term sandwich can
express chirality. This rests on the grade-coupling axis and does not arise for the
orientation-agnostic regression we target.

\begin{figure}[t]
\floatconts
  {fig:hierarchy}
  {\caption{Inclusion of restriction-map families in $\Cl(3,0,0)$ (rotations, $SO(3)$) and
   $\Cl(3,0,1)$ (rigid motions, $SE(3)$). Each edge denotes inclusion ($\subseteq$, lower in
   higher), and each number is the dimension of that family in $\End_{\mathbb{R}}$. Unnormalized sandwich dimensions are shown,
   normalizing fixes the overall scale and removes one dimension
   (reversion $7\!\to\!6$ in $\Cl(3,0,0)$, $16\!\to\!15$ in $\Cl(3,0,1)$). The braced
   families are equivariant, $s_{Rc\rev{R}}(Rx\rev{R})=R\,s_c(x)\,\rev{R}$. Sandwich sums
   fill exactly half of the linear maps; in $\Cl(3,0,0)$ these are the ones commuting with
   the central pseudoscalar.}}
  {\includegraphics[width=\linewidth]{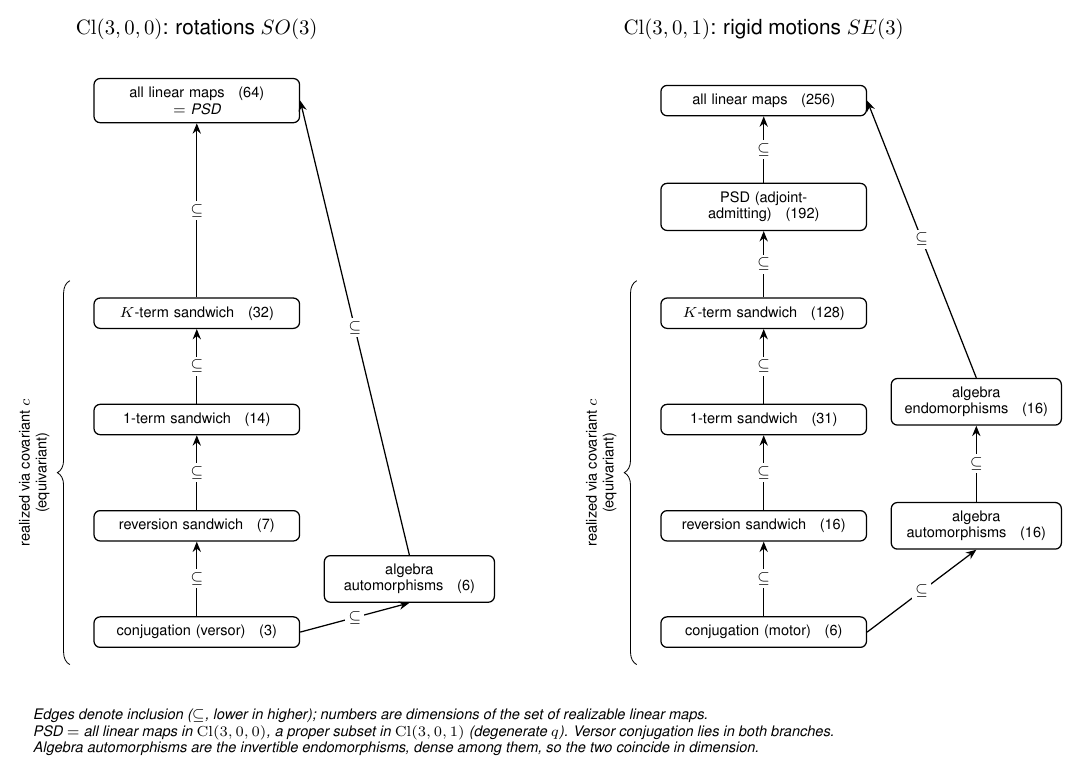}}
\end{figure}

\paragraph{Conditioning.}
Stability separates the family too. In $\Cl(3,0,0)$, write $w=\rev{c}c$ and
$r=\norm{w_{\mathrm{vec}}}/w_0$. Since $w$ lies in grades $\{0,1\}$, the reversion and
$\gamma$ self-terms have condition number $\kappa=\big((1+r)/(1-r)\big)^2$: a closed form in
one ratio, equal to $1$ at the versor limit $r=0$ and growing as $r\to1$. The general
sandwich, built from two independent elements, has no comparable single-parameter handle,
and in the degenerate $\Cl(3,0,1)$ the self-term $w$ carries a pseudoscalar component, so the
single-ratio form no longer applies.

Together these axes place CSNN precisely: the $K{=}1$ normalized reversion sandwich,
grade-confined to $\{0,1\}$ and single-ratio conditioned. It specializes the general family
$\sum_k a_k\,x\,b_k$ to $s_c$ by taking $K{=}1$ and tying the factors by reversion and
normalization ($a=c$, $b=\rev{c}/q(c)$), recovering~\eqref{eq:sandwich}; a versor $c$
collapses it further to conjugation $\Ad_c$. \figureref{fig:hierarchy} summarizes how the
families nest in both signatures.

\section{The Clifford Sheaf Neural Network Layer}
\label{sec:method}

We construct the CSNN layer by implementing the normalized reversion sandwich $s_c$ together with a
self-reference mechanism: the network generates each node's restriction element from that
node's own features (\sectionref{sec:generate}), and the sheaf Laplacian updates each node
against its own transported state (\sectionref{sec:laplacian}).

\subsection{Architecture Overview}
\label{sec:architecture}

CSNN operates on geometric graphs whose nodes carry features and $3$D positions, embedded
directly into the algebra (\figureref{fig:architecture}). Scalar node features occupy grade $0$,
the invariant part, in both signatures. The two signatures differ in how positions enter.
In $\Cl(3,0,0)$ a position occupies grade $1$, transforming as a Euclidean vector. In
$\Cl(3,0,1)$, the projective algebra whose symmetry group is $SE(3)$, we embed a point
$(x,y,z)$ as the grade-$3$ trivector
$e_{123}-x\,e_{023}+y\,e_{013}-z\,e_{012}$, the standard PGA point with homogeneous component
$e_{123}$ fixed to $1$, following \citet{brehmer2023gatr}.

\begin{figure}[t]
\floatconts
  {fig:architecture}
  {\caption{A single CSNN layer. Node multivector features are transported across each
   edge by normalized reversion sandwiches, combined through the sheaf
   Laplacian, and used to update node features and positions, gated by scalar
   invariants.}}
  {\includegraphics[width=\linewidth]{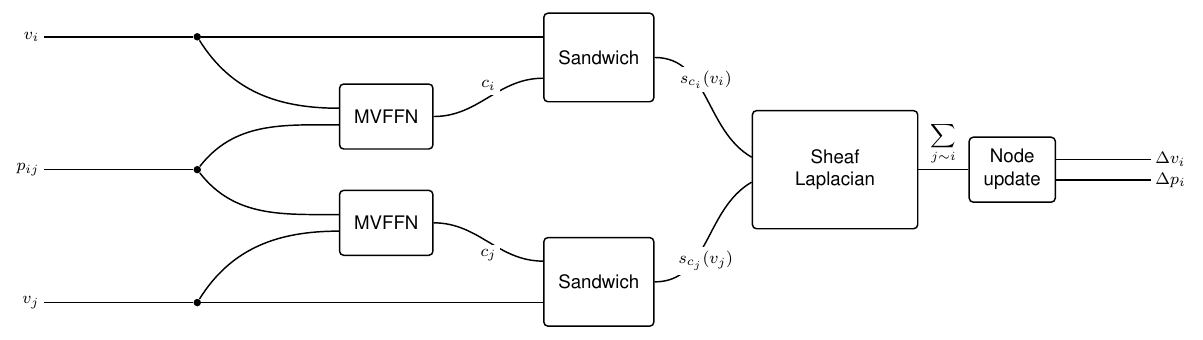}}
\end{figure}

\subsection{Generating the Restriction Element}
\label{sec:generate}

The restriction elements are not free parameters. The layer generates them from the node and edge features by an abstract map $\phi$, so the
transport is both learned and $\Spin$-equivariant. For an edge $\{i,j\}$ we write the relative position
\begin{equation}
  p_{ij} = p_j - p_i,
\end{equation}
and for each endpoint $k\in\{i,j\}$, set
\begin{equation}
  c_k = \phi\big(L_v\,v_k + L_e\,p_{ij}\big),
  \label{eq:generate}
\end{equation}
where $L_v,L_e$ are equivariant multivector linear maps and $\phi:\Cl\to\Cl$ is any
\emph{$\Spin$-equivariant} map, $\phi(R\,x\,\rev{R})=R\,\phi(x)\,\rev{R}$ for every
$R\in\Spin$. Each restriction element is then covariant, $c_k\mapsto R\,c_k\,\rev{R}$, the
condition (Proposition~\ref{prop:equiv}) that makes the transport $s_{c_k}$ equivariant. Each
$c_k$ uses only the node's \emph{own} feature $v_k$ and the edge geometry $p_{ij}$, a
translation-invariant quantity read directly from the configuration. We never form a
cross-stalk feature difference such as $v_j-v_i$: features in different stalks are comparable
only after transport, so a raw difference would be the meaningless coboundary of the trivial
sheaf.

CSNN builds $\phi$ from \emph{multivector neurons} (MVN)~\citep{liu2024mvn}: grade-wise
linear maps~\citep{ruhe2023cgenn} composed around a gated nonlinearity. Every primitive is
$\Spin$-equivariant, since grade-wise linear maps commute with the rotor action and the gate
uses only invariant scalars, so $\phi$ is $\Spin$-equivariant and the same primitives build
the message and node-update maps. \appendixref{app:mvn} gives the explicit block.

\subsection{The Sheaf Laplacian}
\label{sec:laplacian}

CSNN uses the normalized reversion sandwich~\eqref{eq:sandwich} as its restriction map on
\emph{both} sides of each edge. For an edge $e=\{i,j\}$,
\begin{equation}
  f_{i\to e} = s_{c_i},\qquad f_{j\to e} = s_{c_j},
\end{equation}
with restriction elements $c_i,c_j$ produced from the node and edge features
by~\eqref{eq:generate} (\sectionref{sec:generate}). Once an orientation is fixed, the general
coboundary~\eqref{eq:coboundary-generic} specialises to
$(\delta v)_e = s_{c_j}(v_j) - s_{c_i}(v_i)$.
The adjoint of the normalized reversion sandwich under the Clifford scalar
product~\eqref{eq:scalar-product} is
\begin{equation}
  s_c^{\ast}(y) = \rev{c}\,y\,c\,/\,q(c),
  \label{eq:adjoint}
\end{equation}
which yields the boundary
$(\delta^{\ast}\omega)_k=\sum_{e\ni k}[e{:}k]\,s_{c_k}^{\ast}(\omega_e)$, where the
incidence sign $[e{:}k]$ is $+1$ at the head of $e$ and $-1$ at the tail. The incidence
signs cancel in the composition $\Lap=\delta^{\ast}\delta$, so the sheaf Laplacian is
independent of the chosen orientations, and at each node
\begin{equation}
  (\delta^{\ast}\delta v)_i = \sum_{j\sim i} s_{c_i}^{\ast}\!\big(s_{c_i}(v_i) - s_{c_j}(v_j)\big).
  \label{eq:laplacian}
\end{equation}
The layer uses the negative of this Laplacian (``neighbour minus self'') as its message;
\appendixref{app:expanded} writes the self and cross terms out fully as nested sandwiches.
Two special cases anchor the construction. With identity restriction elements,
\eqref{eq:laplacian} reduces to the standard graph Laplacian $\sum_{j\sim i}(v_i-v_j)$; with
versor elements ($c\in\Gamma$), $s_c$ is conjugation, the self term is the identity, and CSNN
recovers a grade-preserving connection Laplacian of the kind used by \citet{barbero2022sheaf}.
For general elements CSNN moves beyond both: the explicit adjoint~\eqref{eq:adjoint} realizes
the Gram form $\Lap=\delta^{\ast}\delta$~\eqref{eq:psd}, so the Laplacian stays exactly
positive semidefinite with no versor constraint even as the transport mixes grades
(\sectionref{sec:family}).

\section{First-Order Representational Placement}
\label{sec:firstorder}

Where does CSNN sit among equivariant models? Grade confinement fixes the answer. Under the
rotation action the four grades of $\Cl(3,0,0)$ carry familiar irreducibles: grades $0$ and
$3$ are spin-$0$ (invariant), grades $1$ and $2$ spin-$1$. CSNN thus operates entirely with
$L{=}0$ and $L{=}1$ representations and carries no rank-two tensors, placing it with
scalar--vector models such as PaiNN~\citep{schutt2021painn} rather than higher-order tensor
models such as Equiformer~\citep{liao2023equiformer}. By Proposition~\ref{prop:confine} this
is structural, not incidental: the reversion CSNN is a first-order model by construction. This
grade analysis is specific to $\Cl(3,0,0)$; in $\Cl(3,0,1)$ reversion preserves the
$\beta$-eigenspaces $\{0,1,4\}$ and $\{2,3\}$, with positions (grade $3$) in the latter.

\section{Related Work}
\label{sec:related}

\paragraph{Clifford and geometric-algebra networks.}
Many architectures compute with multivector features and the geometric product: Clifford
group-equivariant networks \citep{ruhe2023cgenn}, geometric Clifford algebra networks
\citep{ruhe2023gcan}, Clifford neural layers for PDEs \citep{brandstetter2023clifford}, the
geometric algebra transformer \citep{brehmer2023gatr}, and multivector neurons
\citep{liu2024mvn}. These fix the message-passing form or constrain transformations to the
Clifford group; CSNN instead makes edge-wise transport a learned, versor-free restriction
map inside a sheaf Laplacian.

\paragraph{Sheaf neural networks.}
Sheaf neural networks \citep{hansen2020sheaf} and neural sheaf diffusion
\citep{bodnar2022nsd} attach real vector-space stalks to a graph and learn restriction maps;
connection Laplacians \citep{barbero2022sheaf} force these orthogonal, all for node-level
tasks. CSNN gives the stalks Clifford structure and normalized reversion sandwiches
that lift orthogonality while keeping an exact positive-semidefinite Laplacian, and targets
graph-level regression, where the sheaf-Laplacian kernel becomes a genuine constraint
(\sectionref{sec:conclusion}).

\paragraph{Equivariant models for molecules.}
Scalar--vector ($L{=}1$) models such as EGNN \citep{satorras2021egnn} and PaiNN
\citep{schutt2021painn}, and first-order attention models such as the equivariant
transformer \citep{tholke2022torchmdnet}, share CSNN's representational class; the
higher-order tensor models SEGNN \citep{brandstetter2022segnn}, tensor field networks
\citep{thomas2018tfn}, MACE \citep{batatia2022mace}, and Equiformer
\citep{liao2023equiformer} use $L{\ge}2$ features.

\paragraph{Gauge equivariance and topological message passing.}
Reading restriction maps as parallel transport connects CSNN to gauge-equivariant
networks \citep{cohen2019gauge,dehaan2021gauge}, and its cell-poset formulation to
simplicial and topological message passing
\citep{bodnar2021weisfeiler,eijkelboom2023simplicial,liu2024csmpn}.

\section{Conclusion}
\label{sec:conclusion}

CSNN combines Clifford multivector stalks with a sheaf Laplacian, built from normalized reversion sandwiches and positive semidefinite
by construction, with no versor constraint.
Characterizing the sandwich family shows that these maps span exactly half of
$\End_{\mathbb{R}}$, expressivity grows with the number of terms, and
CSNN's reversion member is the first-order, single-ratio-conditioned corner of the family.

A structural caveat concerns the sheaf-Laplacian kernel, where transported configurations
agree and the relational signal vanishes. Since CSNN only approximates diffusion, this bounds the linear core rather than the model.

\bibliographystyle{plainnat}
\bibliography{references}

\appendix

\section{Notation and Definitions}
\label{app:glossary}

For reference, this appendix collects the notation and main definitions used in the paper,
grouped by topic. All conventions match the body.

\begin{longtable}{@{}l p{0.24\linewidth} p{0.51\linewidth}@{}}
\toprule
\textbf{Symbol} & \textbf{Term} & \textbf{Definition}\\
\midrule
\endhead
\multicolumn{3}{@{}l}{\emph{Clifford algebra}}\\*
\cmidrule(r){1-3}
$\Cl(p,q,r)$ & Clifford algebra & Signature $(p,q,r)$ on $\mathbb{R}^n$, $n=p+q+r$: the quotient of the tensor algebra by $v^2=Q(v)$; real dimension $2^n$.\\
$\Cl(3,0,0)$ & Euclidean case & $n=3$, dimension $8$; symmetry group $SO(3)$.\\
$\Cl(3,0,1)$ & projective case, PGA & $n=4$, dimension $16$; symmetry group $SE(3)$; degenerate scalar product with an $8$-dimensional radical.\\
$g_{ij}$ & metric & Defined by $e_ie_j+e_je_i=2g_{ij}$.\\
$\Cl^{(k)}$ & grade-$k$ subspace & Spanned by the $k$-blades, the products of $k$ orthogonal basis vectors; dimension $\binom{n}{k}$.\\
$[x]_k$ & grade projection & The grade-$k$ component of $x$; $[x]_0$ is the scalar part.\\
$\omega$ & pseudoscalar & The top-grade blade: $e_{123}$ in $\Cl(3,0,0)$, the null grade-$4$ blade in $\Cl(3,0,1)$.\\
\addlinespace
\multicolumn{3}{@{}l}{\emph{Involutions and the scalar product}}\\*
\cmidrule(r){1-3}
$\alpha$ & grade involution & Acts as $(-1)^k$ on grade $k$.\\
$\rev{x}=\beta(x)$ & reversion & Acts as $(-1)^{k(k-1)/2}$ on grade $k$; an anti-automorphism.\\
$\gamma$ & Clifford conjugation & The composition $\gamma=\alpha\circ\beta$.\\
$\ip{x}{y}$ & Clifford scalar product & $\ip{x}{y}=\grade{\rev{x}\,y}{0}$; a symmetric bilinear form, positive definite in $\Cl(3,0,0)$ and positive semidefinite in $\Cl(3,0,1)$.\\
$q(x)$ & quadratic form & $q(x)=\ip{x}{x}$.\\
\addlinespace
\multicolumn{3}{@{}l}{\emph{Groups and equivariance}}\\*
\cmidrule(r){1-3}
 & versor & A multivector formed as the geometric product of invertible vectors, each a non-null grade-$1$ element with inverse $v^{-1}=v/q(v)$.\\
$\Gamma$ & Clifford group & The group of versors.\\
$\Spin$ & spin group & The unit even versors; acts by $x\mapsto Rx\rev{R}$ with $\rev{R}R=1$.\\
$\Ad_c$ & conjugation & $x\mapsto c\,x\,c^{-1}$; grade-preserving. $s_c$ reduces to $\Ad_c$ when $c$ is a versor.\\
$SO(3),\,SE(3)$ & rotation, rigid-motion groups & Realized in $\Cl(3,0,0)$ and $\Cl(3,0,1)$ respectively.\\
\addlinespace
\multicolumn{3}{@{}l}{\emph{Restriction maps}}\\*
\cmidrule(r){1-3}
$a\,x\,b$ & sandwich map & The single sandwich; the $K$-term sandwich is $\sum_{k=1}^K a_k\,x\,b_k$.\\
$s_c$ & normalized reversion sandwich & $s_c(x)=c\,x\,\rev{c}/q(c)$, the restriction map CSNN uses; $c$ is the \emph{restriction element} (\sectionref{sec:theory}).\\
$s_c^{\ast}$ & adjoint & $s_c^{\ast}(y)=\rev{c}\,y\,c/q(c)$, the adjoint of $s_c$ under $\ip{\cdot}{\cdot}$.\\
$J$ & pseudoscalar multiplication & $J(x)=e_{123}x$ in $\Cl(3,0,0)$; $J^2=-\mathrm{Id}$.\\
$\End_{\mathbb{R}}$ & endomorphism algebra & Real-linear maps of the stalk; dimension $64$ for $\Cl(3,0,0)$, $256$ for $\Cl(3,0,1)$.\\
$w,\,r,\,\kappa$ & conditioning & For the reversion self-term in $\Cl(3,0,0)$: $w=\rev{c}c$, $r=\norm{w_{\mathrm{vec}}}/w_0$, and $\kappa=((1+r)/(1-r))^2$ (\sectionref{sec:family}).\\
\addlinespace
\multicolumn{3}{@{}l}{\emph{Sheaf machinery}}\\*
\cmidrule(r){1-3}
 & cellular sheaf & Assigns a stalk to each cell and a restriction map to each incident vertex--edge pair (\sectionref{sec:sheaves}).\\
 & stalk & The vector space attached to a cell.\\
$f_{k\to e}$ & restriction map & The linear map from node $k$'s stalk into edge $e$'s stalk.\\
$\delta$ & coboundary & $(\delta v)_e=f_{j\to e}(v_j)-f_{i\to e}(v_i)$.\\
$\Lap$ & sheaf Laplacian & $\Lap=\delta^{\ast}\delta$; positive semidefinite by construction.\\
 & Dirichlet energy & $\ip{v}{\Lap v}=\norm{\delta v}^2$; the disagreement between node features after transport.\\
\addlinespace
\multicolumn{3}{@{}l}{\emph{Layer components}}\\*
\cmidrule(r){1-3}
MVN & multivector neuron & Multivector-valued unit of \citet{liu2024mvn}: a grade-preserving linear map with a gated nonlinearity.\\
MVFFN & feed-forward block & CSNN's choice for the generator $\phi$: an MVN-MLP-style block $L_2(\sigma(L_1 x))$ (\sectionref{sec:generate}).\\
$L$ & multivector linear map & One shared scalar weight per grade, no bias; grade-preserving and equivariant.\\
$L{=}0,\,L{=}1$ & spin-$0$, spin-$1$ & The irreducibles CSNN uses; the model is first-order (\sectionref{sec:firstorder}).\\
\bottomrule
\end{longtable}

\section{Proofs}
\label{app:proofs}

\begin{proof}[Proof of Proposition~\ref{prop:equiv}]
Since $\rev{R}R=1$, each term satisfies
\[
  (Ra_k\rev{R})(Rx\rev{R})(Rb_k\rev{R}) = R\,a_k x b_k\,\rev{R},
\]
and summing over $k$ preserves the identity. For $s_c$, reversion is an anti-automorphism, so
$\rev{Rc\rev{R}}=R\,\rev{c}\,\rev{R}$, and $q$ is $\Spin$-invariant, $q(Rc\rev{R})=q(c)$; hence
\[
  \begin{aligned}
    s_{Rc\rev{R}}(Rx\rev{R})
    &= Rc\rev{R}\,Rx\rev{R}\,R\rev{c}\rev{R}/q(c)\\
    &= R\,c\,x\,\rev{c}\,\rev{R}/q(c) = R\,s_c(x)\,\rev{R}.
  \end{aligned}
\]
\end{proof}

\begin{proof}[Proof of Proposition~\ref{prop:confine}]
For $v$ with $\beta(v)=\pm v$, $\beta(c\,v\,\rev{c})=\beta(\rev{c})\beta(v)\beta(c)
= c(\pm v)\rev{c}=\pm\,c\,v\,\rev{c}$, so $s_c(v)$ lies in the same eigenspace; composition
preserves the property.
\end{proof}

\begin{proof}[Proof of Proposition~\ref{prop:clin}]
Since $e_{123}$ is central, $a(e_{123}x)b=e_{123}(axb)$, so $x\mapsto axb$ commutes with $J$,
and so does any sum of such maps. Since $\alpha(e_{123})=-e_{123}$, the map
$x\mapsto a\,\alpha(x)\,b$ anticommutes with $J$ instead.
\end{proof}

\paragraph{$\alpha$ is not a finite sandwich sum.}
For $T\in\End(\Cl)$ the functional $\operatorname{tr}(\alpha\,T)$ is linear in $T$. A short
computation shows that, for a single sandwich $S:x\mapsto a\,x\,b$,
$\operatorname{tr}(\alpha\,S)$ depends only on the pseudoscalar coefficients of $a$ and $b$:
it is a multiple of $a_\omega b_\omega$ (with $\omega$ the pseudoscalar), and that multiple
vanishes when $n$ is odd and is proportional to $q(\omega)$ when $n$ is even. In both our
signatures it is therefore $0$: $\Cl(3,0,0)$ has odd $n=3$, and in $\Cl(3,0,1)$ the
pseudoscalar is null. By linearity the functional vanishes on every sandwich sum, whatever
its length; but $\operatorname{tr}(\alpha\,\alpha)=\operatorname{tr}(\mathrm{Id})=\dim\Cl\neq0$.
Hence no finite sum $\sum_k a_k\,x\,b_k$ can equal $\alpha$, in either signature.

\section{Expanding the Sheaf Laplacian in the Restriction Elements}
\label{app:expanded}

Substituting the normalized reversion sandwich $s_c(x)=c\,x\,\rev{c}/q(c)$ and its adjoint
$s_c^{\ast}(y)=\rev{c}\,y\,c/q(c)$ into the node Laplacian~\eqref{eq:laplacian} yields the
nested-sandwich form~\eqref{eq:expanded}. On an edge $\{i,j\}$ the two transports are
\begin{equation*}
  s_{c_i}(v_i)=\frac{c_i\,v_i\,\rev{c}_i}{q(c_i)}, \qquad
  s_{c_j}(v_j)=\frac{c_j\,v_j\,\rev{c}_j}{q(c_j)}.
\end{equation*}
Apply the adjoint $s_{c_i}^{\ast}(y)=\rev{c}_i\,y\,c_i/q(c_i)$ to each, and use associativity
of the geometric product to drop the inner brackets:
\begin{align*}
  s_{c_i}^{\ast}\!\big(s_{c_i}(v_i)\big)
    &= \frac{\rev{c}_i\,\big(c_i\,v_i\,\rev{c}_i\big)\,c_i}{q(c_i)\,q(c_i)}
     = \frac{\rev{c}_i\,c_i\,v_i\,\rev{c}_i\,c_i}{q(c_i)^2},\\[2pt]
  s_{c_i}^{\ast}\!\big(s_{c_j}(v_j)\big)
    &= \frac{\rev{c}_i\,\big(c_j\,v_j\,\rev{c}_j\big)\,c_i}{q(c_i)\,q(c_j)}
     = \frac{\rev{c}_i\,c_j\,v_j\,\rev{c}_j\,c_i}{q(c_i)\,q(c_j)}.
\end{align*}
Because $s_{c_i}^{\ast}$ is linear,
$s_{c_i}^{\ast}\!\big(s_{c_i}(v_i)-s_{c_j}(v_j)\big)$ is the difference of these two terms.
Summing over the neighbours $j\sim i$ gives the nested-sandwich form
\begin{equation}
  (\delta^{\ast}\delta v)_i = \sum_{j\sim i}
    \left[\frac{\rev{c}_i\,c_i\,v_i\,\rev{c}_i\,c_i}{q(c_i)^2}
        - \frac{\rev{c}_i\,c_j\,v_j\,\rev{c}_j\,c_i}{q(c_i)\,q(c_j)}\right].
  \label{eq:expanded}
\end{equation}

\section{The Multivector-Neuron Block}
\label{app:mvn}

For $\phi$ we take a multivector feed-forward block, MVFFN, an MVN-MLP-style composition of
two grade-wise linear maps around a gated nonlinearity $\sigma$, with $L_1$ widening and $L_2$
projecting,
\begin{equation}
  \mathrm{MVFFN}(x)=L_2\big(\sigma(L_1 x)\big).
\end{equation}
A grade-wise linear map applies one shared scalar weight per grade with no bias and commutes
with the rotor action. For $\sigma$ we use a sigmoid-smoothed form of the MVN rejection
nonlinearity: with $q=L_q x$ and $k=L_k x$,
\begin{equation}
  \sigma(x)=\tau\,q+(1-\tau)\Big(q-\tfrac{\ip{q}{k}}{\norm{k}^2}\,k\Big),
  \qquad \tau=\mathrm{sigmoid}(\ip{q}{k}).
\end{equation}
The gates $\ip{q}{k}$ and $\norm{k}^2$ are invariant scalars multiplying equivariant
multivectors, so $\sigma$ is equivariant. Sums and compositions of equivariant maps stay
equivariant, so $\phi=\mathrm{MVFFN}$ is $\Spin$-equivariant.

\end{document}